\documentclass[journal]{IEEEtran}
\usepackage{amsmath,amsfonts}
\usepackage{algorithmic}
\usepackage{algorithm}
\usepackage{array}
\usepackage{textcomp}
\usepackage{stfloats}
\usepackage{url}
\usepackage{verbatim}
\usepackage{caption}
\usepackage{cite}
\usepackage{amsmath, amssymb, amsfonts, amsthm}
\usepackage{hyperref}
\usepackage{url}
\usepackage[T1]{fontenc}
\usepackage{amsmath, amssymb, amsfonts, amsthm}
\usepackage[caption=false,font=normalsize,labelfont=sf,textfont=sf]{subfig}
\usepackage{graphicx}
\newtheorem{theorem}{Theorem}[section]
\newtheorem{proposition}[theorem]{Proposition} 
\usepackage{float}
 
\usepackage[utf8]{inputenc} 
\usepackage{hyperref}       
\usepackage{url}            
\usepackage{booktabs}       
\usepackage{amsfonts}       
\usepackage{nicefrac}       
\usepackage{microtype}      
\usepackage{xcolor}         
\usepackage{multirow}
\begin{document}

\title{Spatial Lifting for Dense Prediction}

\author{Mingzhi Xu, Yizhe Zhang
\thanks{M. Xu and Y. Zhang are with the School of Computer Science and Engineering, Nanjing University of Science and Technology, Nanjing, China. Email: zhangyizhe@njust.edu.cn

Preprint, Under Review.}

}

\maketitle

\begin{abstract}
We present Spatial Lifting (SL), a novel methodology for dense prediction tasks. SL operates by lifting standard inputs, such as 2D images, into a higher-dimensional space and subsequently processing them using networks designed for that higher dimension, such as a 3D U-Net. Counterintuitively, this dimensionality lifting allows us to achieve good performance on benchmark tasks compared to conventional approaches, while reducing inference costs and significantly lowering the number of model parameters.  The SL framework produces intrinsically structured outputs along the lifted dimension. This emergent structure facilitates dense supervision during training and enables robust, near-zero-additional-cost prediction quality assessment at test time. We validate our approach across 19 benchmark datasets—13 for semantic segmentation and 6 for depth estimation—demonstrating competitive dense prediction performance while reducing the model parameter count by over 98\% (in the U-Net case) and lowering inference costs. Spatial Lifting introduces a new vision modeling paradigm that offers a promising path toward more efficient, accurate, and reliable deep networks for dense prediction tasks in vision.
\end{abstract}


\section{Introduction}
\IEEEPARstart{D}{ense} prediction tasks, such as semantic segmentation \cite{Ronneberger2015_UNet, Chen2017_DeepLab}, depth estimation \cite{ranftl2020towards, Bhat2021_AdaBins}, and optical flow \cite{Teed2020_RAFT}, form a cornerstone of computer vision with critical applications ranging from autonomous driving and robotics to medical image analysis and augmented reality. The goal in these tasks is to produce a structured output, assigning a label or value to every pixel or voxel in the input. Achieving high accuracy, efficiency and robustness in these predictions is paramount for their reliable deployment in real-world scenarios. 

Significant progress in deep learning has led to highly effective models for dense prediction, often leveraging sophisticated convolutional neural network (CNN) architectures \cite{He2016_ResNet, Long2015_FCN} or, more recently, transformer-based approaches \cite{Dosovitskiy2021_ViT, Xie2021_SegFormer}. While these methods attain state-of-the-art performance, this success often comes at the cost of substantial computational complexity. State-of-the-art models frequently involve vast numbers of parameters and high floating-point operations, posing challenges for deployment on resource-constrained platforms like mobile devices or embedded systems. Furthermore, obtaining reliable estimates of prediction quality or uncertainty from these complex models typically requires additional mechanisms, such as model ensembling or Monte Carlo dropout \cite{Gal2016_DropoutUncertainty}, which can further increase computational overhead during inference.

\begin{figure}[t]
    \centering
    \includegraphics[width=0.48\textwidth]{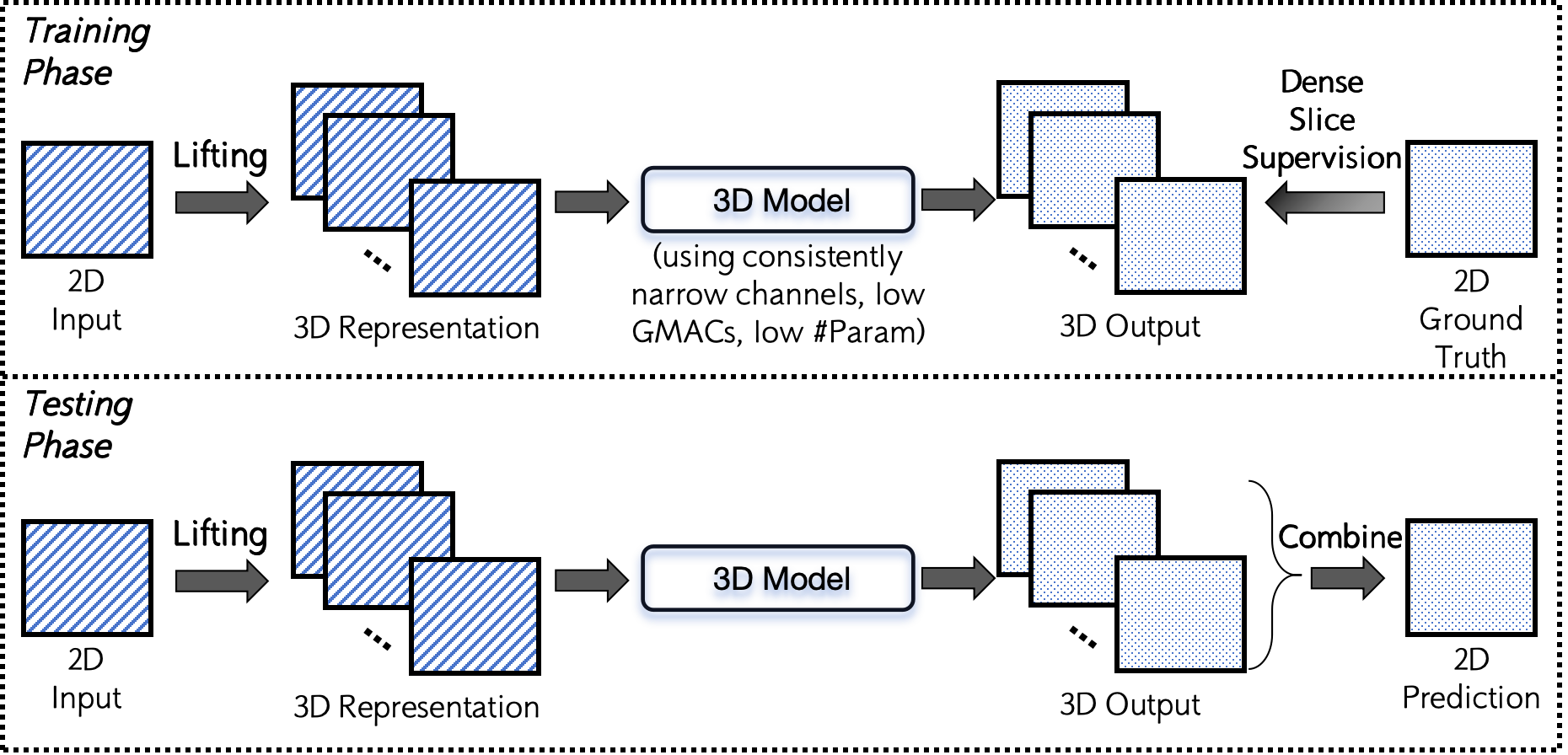}
    \caption{High-level view of Spatial Lifting for Dense Prediction.}
    \label{fig:overview}
\end{figure}

Most efforts to reduce the computational cost of deep learning, such as efficient architecture design (e.g.,\cite{howard2017mobilenets}, \cite{zhang2018shufflenet}, \cite{tan2019efficientnet}), post-training pruning (e.g.,~\cite{han2015learning}), and quantization (e.g.,~\cite{krishnamoorthi2018quantizing}), knowledge distillation (e.g.,~\cite{touvron2021training}) have operated directly on the native input dimensions. However, these methods have not explored whether modeling in a higher spatial dimension could offer additional performance benefits. In this work, we present \textbf{Spatial Lifting (SL)}, an effective and efficient methodology for dense prediction. The core idea of SL is simple yet powerful: standard inputs, such as 2D images, are first ``lifted" into a higher-dimensional space (e.g., 3D). This lifted image (or representation) is then processed using deep networks designed for that higher dimension – for instance, employing a 3D U-Net \cite{Cicek2016_3DUNet} architecture to process the 3D image. We demonstrate that this deliberate dimensionality expansion allows us to achieve better performance on established benchmark tasks compared to conventional architectures operating in the original input dimension. Crucially, this performance gain is achieved while concurrently \textit{significantly reducing} the number of model parameters and \textit{lowering} the GMACs required for inference. Beyond computational efficiency and accuracy, the SL framework possesses another unique characteristic: it intrinsically produces structured outputs along the auxiliary lifted dimension without requiring additional computation. This emergent structure enables highly beneficial, facilitating advanced dense supervision strategies during the training phase. Moreover, at test time, it enables robust estimation of prediction quality or uncertainty with near zero additional computational overhead, offering valuable insights into model reliability on a per-prediction basis. A high-level overview of the SL is illustrated in Fig.~\ref{fig:overview}. Our key contributions are summarized as follows.
\IEEEpubidadjcol
\begin{itemize}
\item {We introduce Spatial Lifting (SL)}, a novel paradigm for dense prediction tasks that lifts 2D inputs into a higher-dimensional (e.g., 3D) space and processes them using standard higher-dimensional networks, enabling richer spatial modeling at reduced computational cost.

\item {SL significantly reduces model complexity}, achieving competitive or superior performance with dramatically fewer parameters and GMACs compared to conventional architectures such as UNet and PVT-based decoders.

\item {A built-in prediction quality assessment (PQA) mechanism is proposed} that leverages internal consistency across output slices to estimate segmentation quality with negligible overhead, showing good correlation with true performance metrics.

\item {Theoretical insights are provided} showing that SL introduces an implicit spatial regularization effect along the lifted dimension, improving generalization and robustness via smoother learned representations.

\item {Comprehensive experiments} on 13 semantic segmentation and 6 depth estimation datasets validate the effectiveness, efficiency, and generalizability of SL across tasks, network backbones, and decoders.
\end{itemize}

\section{Related Work}
Our proposed Spatial Lifting (SL) methodology intersects with several research areas, including efficient deep learning architectures, techniques for enhancing spatial reasoning in dense prediction, methods leveraging dimensionality changes, and uncertainty estimation. We position SL relative to these domains below.

\subsection{Dense Prediction Architectures}
Modern dense prediction heavily relies on deep neural networks, primarily Fully Convolutional Networks (FCNs) \cite{Long2015_FCN} and U-Net architectures \cite{Ronneberger2015_UNet, Cicek2016_3DUNet}. These often employ an encoder-decoder structure to capture multi-scale contextual information while preserving spatial resolution. Subsequent work, such as DeepLab variants \cite{Chen2017_DeepLab}, introduced techniques like atrous (dilated) convolutions and spatial pyramid pooling \cite{Zhao2017_PSPNet, He2015_SPPNet} to enlarge receptive fields without sacrificing resolution. More recently, Vision Transformers (ViTs) \cite{Dosovitskiy2021_ViT} and their derivatives (e.g., SegFormer \cite{Xie2021_SegFormer}, SETR \cite{Zheng2021_SETR}) have shown strong performance by leveraging global self-attention mechanisms. While powerful, these methods typically achieve high accuracy through increased network depth, width, or module complexity, often leading to significant computational demands. SL offers a different path: instead of adding complexity within the native input dimension, it lifts the image to a higher dimension, aiming to capture complex spatial relationships more efficiently within that space using potentially simpler base operations (e.g., standard 3D convolutions).

\subsection{Efficient Deep Learning}
Addressing the significant computational cost of modern deep learning models has been a major focus of research. Efforts can be broadly grouped into two main areas: designing efficient network architectures from the ground up and compressing large, pre-existing models.

A primary strategy involves creating inherently efficient architectures. Landmark examples include {MobileNets} \cite{howard2017mobilenets}, which utilize depthwise separable convolutions to reduce FLOPs, and {ShuffleNets} \cite{zhang2018shufflenet}, which employ pointwise group convolutions and channel shuffling for further efficiency. More systematic approaches like {EfficientNet} \cite{tan2019efficientnet} have introduced compound scaling, a principled method to balance network depth, width, and resolution to create highly efficient and accurate models. These methods primarily focus on optimizing 2D convolutional operations or finding efficient layer combinations within the original data dimensionality.

Another significant line of research focuses on model compression and acceleration. Techniques such as {network pruning} \cite{han2015learning, frankle2018lottery} aim to remove redundant parameters from a trained network, either in an unstructured or structured manner. {Quantization} reduces the numerical precision of weights and activations (e.g., from 32-bit floats to 8-bit integers), significantly shrinking model size and often speeding up inference. {Knowledge distillation} \cite{hinton2015distilling} transfers knowledge from a large "teacher" model to a smaller "student" model, enabling the smaller model to achieve higher accuracy than it would if trained alone. Furthermore, {Neural Architecture Search (NAS)} \cite{liu2018darts, tan2019mnasnet, cai2018proxylessnas} has automated the discovery of novel, efficient architectures tailored for specific hardware and tasks.

SL contrasts with these by fundamentally changing the \textit{dimensionality} of the space where primary computation occurs. Instead of optimizing 2D operations, SL leverages potentially simpler, standard convolutions (e.g., 3D) in a higher-dimensional space, achieving efficiency by processing a representation where spatial relationships might be more explicitly captured, rather than solely through intricate 2D filter designs or channel manipulations.

\subsection{Dimensionality Manipulation}
{Depth-to-Space / Space-to-Depth (Pixel Shuffle):} The "Depth-to-Space" (D2S) operation, also known as Pixel Shuffle \cite{shi2016real}, rearranges data from the channel dimension into spatial dimensions, commonly used for efficient upsampling in super-resolution. Its inverse, Space-to-Depth, performs the opposite rearrangement for downsampling. D2S reshapes a tensor $(H, W, C \cdot r^2)$ to $(H \cdot r, W \cdot r, C)$. This operation focuses on rearranging features learned by preceding layers. While both D2S and SL involve interplay between channel and spatial dimensions, their goals and mechanisms differ fundamentally. D2S \textit{rearranges} existing feature dimensions for up/down-sampling, relying on prior layers to encode information correctly in the channels. SL, conversely, \textit{adds} a new spatial dimension and performs the core feature extraction \textit{within} this higher-dimensional space.

{Lifting in Computational Geometry:} Conceptually, SL draws inspiration from ``lifting'' techniques in computational geometry \cite{de2000computational}. For example, lifting 2D points onto a 3D paraboloid transforms Voronoi diagram computation into finding a 3D lower envelope. Such techniques typically map problems to abstract higher-dimensional spaces where geometric structures become simpler or computationally easier to handle. While sharing the principle of increasing dimensionality, SL differs significantly: (1) Its purpose is to \textit{enrich} spatial feature representation within a learning framework, not to simplify a geometric problem. (2) The lifted space in SL remains a \textit{spatial} domain where learnable operations (convolutions) are applied, rather than an abstract geometric construct.

\subsection{Uncertainty and Quality Estimation}
Providing reliable estimates of prediction uncertainty or quality is vital for deploying models in critical applications. Common approaches include Monte Carlo dropout \cite{Gal2016_DropoutUncertainty}, forming ensembles of models, or developing explicitly Bayesian neural networks, all of which typically introduce significant computational overhead at inference time. Recent work also explores learning dedicated uncertainty prediction branches or using test-time augmentation. SL offers a different perspective by leveraging the \textit{intrinsic structure} generated along the lifted dimension. As explored in our work, variations or consistencies along this auxiliary dimension can serve as a proxy for prediction confidence or quality, enabling robust estimation with negligible computational cost compared to methods requiring multiple forward passes or dedicated uncertainty modules. This built-in capability distinguishes SL from post-hoc or computationally intensive uncertainty techniques.

In summary, Spatial Lifting presents a novel approach distinct from existing methods. It leverages dimensionality increase not for geometric simplification or data rearrangement, but as a core mechanism for efficient and powerful spatial feature learning within deep networks, offering unique benefits for supervision and uncertainty estimation.

\section{Methodology}
\subsection{Method Pipeline}
\noindent\textbf{Input.} The proposed SL technique is applicable to both image inputs and feature inputs, such as those generated by a pretrained transformer-based encoder. For clarity and simplicity, the following explanation will utilize image inputs to illustrate the method.

Given a $k$-dimensional ($k$-D) image, for example, when $k$ equals 2,  $\mathbf{I} \in \mathbb{R}^{w \times h \times 3}$ with spatial dimensions $w \times h$ and 3 channels, we first lift the image to a $(k+1)$-dimensional ($(k+1)$-D) space by replicating it $m$ times along a new axis. Here, $m$ is a constant (e.g., $m=16$). The resulting $(k+1)$-D image $\mathbf{I}' \in \mathbb{R}^{w \times h \times m \times 3}$ is then used as the input for the subsequent network. This lifting operation can be defined as:
\begin{equation}
\mathbf{I}'(x, y, z, c) 
= \mathbf{I}(x, y, c), 
\quad \forall z \in \{1, 2, \dots, m\},
\label{eq:example}
\end{equation}
where $x \in \{1, \dots, w\}$, $y \in \{1, \dots, h\}$, and $c \in \{1, 2, 3\}$. 

\noindent\textbf{Forward Pass.} The lifted $(k+1)$-D image $\mathbf{I}'$ is fed into a $(k+1)$-D deep network (e.g., UNet~\cite{Ronneberger2015_UNet}), denoted as $f_\theta$, where $\theta$ represents the trainable parameters of the network. For example, when $k$ equals 2, the input is a 3D image, and the UNet operates in $(k+1)$-D (3D) space, outputting a segmentation prediction $\mathbf{P} \in \mathbb{R}^{w \times h \times m \times C_{out}}$, where $C_{out}$ is the number of segmentation classes. Formally, the network performs:
\begin{equation}
\mathbf{P} = f_\theta(\mathbf{I}').
\end{equation}
The output of the network, logits maps $\mathbf{P}$, contains $m$ slices along the additional dimension. 

\noindent\textbf{Model Training with Dense Slice Supervision.} For training, each slice of the output $\mathbf{P}(x, y, z, c)$ is supervised using the $k$-D ground truth mask $\mathbf{M} \in \mathbb{R}^{w \times h \times C}$. For example, when $k$ equals 2, the ground truth mask is a 2D segmentation mask replicated $m$ times along the new dimension to match the shape of $\mathbf{P}$:
\begin{equation}
\mathbf{M}'(x, y, z, c) = \mathbf{M}(x, y, c), \quad \forall z \in \{1, 2, \dots, m\}.
\end{equation}
The model is trained using a loss function $\mathcal{L}$, such as the Dice loss, cross-entropy loss, computed slice-wise between $\mathbf{P}$ and $\mathbf{M}'$:
\begin{equation}
\mathcal{L}(\mathbf{P}, \mathbf{M}') = \frac{1}{m} \sum_{z=1}^m \ell(\mathbf{P}(\cdot, \cdot, z, \cdot), \mathbf{M}(\cdot, \cdot, \cdot)).
\end{equation}

Given $n$ training images $\{\mathbf{I}^{(i)}, \mathbf{M}^{(i)}\}_{i=1}^n$, we optimize $\theta$ by minimizing the average loss across all training samples:
\begin{equation}
\mathcal{L}_{\text{train}} = \frac{1}{n} \sum_{i=1}^n \mathcal{L}(f_\theta(\mathbf{I}'^{(i)}), \mathbf{M}'^{(i)}).
\end{equation}

At the end of training, during the final epoch, we compute the average loss $\mathcal{L}_z$ for each slice $z$ for all training samples:
\begin{equation}
\mathcal{L}_z = \frac{1}{n} \sum_{i=1}^n \ell(\mathbf{P}^{(i)}(\cdot, \cdot, z, \cdot), \mathbf{M}^{(i)}(\cdot, \cdot, \cdot)).
\end{equation}
The $s$ slices $z_1, z_2, \dots, z_s$ with the lowest average loss values are selected:
\begin{align}
\{z_1, z_2, \dots, z_s\} 
&= \arg\min_{\{z_1, z_2, \dots, z_s\}} 
    \sum_{k=1}^s \mathcal{L}_{z_k}, \nonumber \\
&\quad z_k \in \{1, 2, \dots, m\},
\end{align}
where $z_1, z_2, \dots, z_s$ are indices corresponding to the selected slices. The default value of $s$ is set to 5.

\noindent\textbf{Model Testing.} During testing, a $k$-D image $\mathbf{I}_{\text{test}} \in \mathbb{R}^{w \times h \times 3}$, here $k$ equals 2, is lifted to $3$-D space, resulting in $\mathbf{I}'_{\text{test}} \in \mathbb{R}^{w \times h \times m \times 3}$. The trained network $f_\theta$ is used to generate predictions:
\begin{equation}
\mathbf{P}_{\text{test}} = f_\theta(\mathbf{I}'_{\text{test}}) \in \mathbb{R}^{w \times h \times m \times C}.
\end{equation}

The segmentation result is obtained by summing the prediction logits of the selected slices \( z_1, z_2, \ldots, z_s \):  
\begin{equation}
\mathbf{S}_{\text{test}}(x, y, c) = \sigma\left(\sum_{z \in \{z_1, z_2, \ldots, z_s\}} \mathbf{P}_{\text{test}}(x, y, z, c)\right),
\end{equation}
where \( \sigma \) denotes the activation function, e.g., the sigmoid function.  The final segmentation is determined by assigning the class with the highest summed probability to each pixel:  
\begin{equation}
\text{Segmentation}(x, y) = \arg\max_c \, \mathbf{S}_{\text{test}}(x, y, c).
\end{equation}

\subsection{Prediction Quality Assessment}

The proposed method can also be utilized to assess the quality of the prediction outputs by analyzing the differences between the selected output slices and unselected output slices. The core idea is to compare the prediction results of selected slices, $\{z_1, z_2, \dots, z_s\}$, with those of the unselected slices, $\{z' \mid z' \in \{1, 2, \dots, m\} \setminus \{z_1, z_2, \dots, z_s\}\}$. For segmentation task, the similarity is measured using a metric such as the Dice coefficient. A higher average similarity between selected and unselected slices implies higher prediction quality, while a lower similarity indicates lower prediction quality. 

Let $\mathbf{PB}(\cdot, \cdot, z, \cdot)$ represent the predicted binary segmentation map at slice $z$. We compute the Dice similarity between the segmentation maps of a selected slice $z_i \in \{z_1, z_2, \dots, z_s\}$ and an unselected slice $z' \notin \{z_1, z_2, \dots, z_s\}$ as: $\text{Dice}(\mathbf{BM}(\cdot, \cdot, z_i, \cdot), \mathbf{BM}(\cdot, \cdot, z', \cdot))$. The prediction quality score is computed as the average of the Dice similarities across all pairs of selected and unselected slices:
\begin{multline}
Q = \frac{1}{s \cdot (m - s)} 
\sum_{z_i \in \{z_1, z_2, \dots, z_s\}} 
\sum_{z' \notin \{z_1, z_2, \dots, z_s\}} \\
\text{Dice}(\mathbf{BM}(\cdot, \cdot, z_i, \cdot), 
\mathbf{BM}(\cdot, \cdot, z', \cdot)).
\end{multline}

Here, $Q$ serves as an indicator of prediction quality:
\underline{High $Q$ Value}: Indicates that the selected slices are highly consistent with the unselected slices, suggesting higher prediction quality.
\underline{Low $Q$ Value}: Indicates greater dissimilarity between selected and unselected slices, suggesting lower prediction quality. This prediction quality assessment framework provides a quantitative measure for assessing the quality of segmentation outputs. Section~\ref{sec:SLU} empirically demonstrates that $Q$ values correlate reasonably well with actual Dice scores.

\subsection{Theoretical Analysis}
This section provides a formal analysis justifying its efficacy, integrating theoretical arguments regarding representation learning, regularization, generalization, and intrinsic quality assessment. We assume the input is $k$-dimensional and lifted to $(k+1)$-dimensions.

\subsubsection{Implicit Regularization}
\begin{theorem}[Spatial Regularization as an Implicit Lipschitz Constraint]
\label{thm:lipschitz}
A $(k+1)$-D convolutional layer with kernel $K$ (size $K_z$ along $z$), operating on the lifted input $\mathbf{I}'$, implicitly enforces a Lipschitz constraint~\cite{khromov2024some} along the lifted dimension $z$. The shared weights of the kernel across $z$ ensure that the change in the layer's output is bounded with respect to changes along the $z$-dimension of its input.
\end{theorem}

\begin{proof}[Proof Argument]
Consider two input tensors to the convolutional layer, $\mathbf{X}_1, \mathbf{X}_2 \in \mathbb{R}^{w' \times h' \times m \times C_{prev}}$, which might differ along the $z$-dimension (due to processing in prior layers). Let $\mathbf{Y}_1, \mathbf{Y}_2$ be the corresponding outputs. The $(k+1)$-D convolution involves applying the same kernel weights $K$ at each position $(x,y,z)$. The difference in output $\lVert \mathbf{Y}_1 - \mathbf{Y}_2 \rVert$ can be bounded by a function of the difference in input $\lVert \mathbf{X}_1 - \mathbf{X}_2 \rVert$ and the norm of the kernel $K$. Specifically, for the $L_1$ or $L_\infty$ norm, the Lipschitz constant along the $z$-dimension is related to the sum of absolute values of the kernel weights along $z$. The weight sharing across $z$ is the mechanism enforcing this constraint. 
\end{proof}

\begin{theorem}[Improved Generalization Bound]
\label{thm:generalization}
Under standard assumptions on data distribution and loss function regularity, the implicit Lipschitz constraint~\cite{khromov2024some} induced by spatial lifting (Theorem \ref{thm:lipschitz}) introduces a smoothness bias along the lifted dimension. This bias potentially leads to improved generalization bounds (e.g., via lower Rademacher complexity~\cite{bartlett2005local}) compared to a $k$-D network of similar parameter capacity lacking this specific inductive bias.
\end{theorem}

\begin{proof}[Proof Argument]
Generalization bounds often depend on measures of model complexity, such as Rademacher complexity or norms of weights. Lipschitz constraints on the function class implemented by the network can lead to tighter bounds, as they restrict how rapidly the function's output can change with input variations. The spatial lifting imposes such a constraint specifically along the $z$-dimension, which effectively reduces the hypothesis space to functions that are stable across the replicated views. Compared to a $k$-D network with the same number of parameters (capacity) but potentially less structured internal representations, the lifted network's inductive bias towards $z$-smoothness can result in better generalization.
\end{proof}

\subsubsection{Effect of Dense Slice Supervision}

The training objective minimizes the average loss over all $m$ slices. This setup provides a regularization effect.

\begin{proposition}[Implicit Regularization]
\label{prop:regularization}
The dense supervision strategy, where $\mathcal{L} = \frac{1}{m} \sum_{z=1}^m \ell(\mathbf{P}_z, \mathbf{M})$, acts as an implicit regularizer by averaging gradients across multiple internal representations (slices), thereby promoting solutions that are less sensitive to noise or specific input variations represented implicitly across the lifted dimension.
\end{proposition}

\begin{proof}[Proof Argument]
The gradient used for updating weights $\theta$ is $\nabla_\theta \mathcal{L} = \frac{1}{m} \sum_{z=1}^m \nabla_\theta \ell(\mathbf{P}_z, \mathbf{M})$. This gradient is an average of the gradients computed independently for each slice's prediction $\mathbf{P}_z$ with respect to the common target $\mathbf{M}$.

Averaging gradients is a known technique for variance reduction in stochastic optimization. In this context, any idiosyncrasies or noise in the computation path leading to a specific slice $\mathbf{P}_z$ (which might arise from the specific filter interactions along the $z$-dimension) that results in a large gradient component $\nabla_\theta \ell(\mathbf{P}_z, \mathbf{M})$ for that slice will have its impact on the final update direction $\nabla_\theta \mathcal{L}$ dampened by averaging with the gradients from the other $m-1$ slices.

This encourages the network to learn weights $\theta$ that perform well {on average} across all slices. Such solutions are less likely to rely on spurious correlations or noise that might be effective only for a subset of the internal processing paths (represented by slices $z$). This behavior is analogous to ensemble methods, where averaging predictions or gradients from multiple models (or multiple views, in this case) leads to more robust and generalized performance. By enforcing consistency towards the target $\mathbf{M}$ across all $m$ slices through a single shared set of weights $\theta$, the training process implicitly regularizes the model, reducing its effective capacity to overfit compared to a standard k-D network trained only on the original k-D input/output pair.
\end{proof}

\subsubsection{Mechanism of Prediction Quality Assessment}

The architecture inherently produces multiple predictions ($\mathbf{P}_z$ for $z=1, \dots, m$) for a single input. Their consistency can be used to estimate output quality. Let $S_{best} = \{z_1, \dots, z_s\}$ be the indices of the $s$ slices with lowest average loss during training, and $S_{rest} = \{1, \dots, m\} \setminus S_{best}$. The PQA score $Q$ measures the average similarity (e.g., Dice) between slices in $S_{best}$ and $S_{rest}$.

\begin{proposition}[Uncertainty Correlation]
\label{prop:sqa}
The prediction quality score $Q$, defined as the average similarity between predictions from slices in $S_{best}$ and $S_{rest}$, serves as an indicator of model confidence or prediction quality. Lower values of $Q$ correlate with higher model uncertainty or difficulty in segmenting the given input.
\end{proposition}

\begin{proof}[Proof Argument]
The training objective drives all slice predictions $\mathbf{P}_z$ towards the ground truth $\mathbf{M}$. For a test input $\mathbf{I}_{\text{test}}$ that is similar to the training data distribution and presents unambiguous features, a well-trained network $f_\theta$ is expected to produce consistent predictions across all slices: $\mathbf{P}_z(\mathbf{I}'_{\text{test}}) \approx \mathbf{P}_{z'}(\mathbf{I}'_{\text{test}})$ for all $z, z' \in \{1, \dots, m\}$. In this case, the similarity between slices in $S_{best}$ and $S_{rest}$ will be high, leading to a high value of $Q$.

Conversely, consider a test input that is challenging (e.g., out-of-distribution, noisy, ambiguous object boundaries). The network $f_\theta$ might struggle to produce a stable, consistent output. Due to the $(k+1)$-D processing, slight variations in how features are integrated along the $z$-dimension could lead to divergent predictions $\mathbf{P}_z$ across different slices. The slices identified as $S_{best}$ represent those that, on average over the training set, performed well. However, for a specific challenging test case, these slices might not agree well with the other slices in $S_{rest}$. This divergence indicates that the network is not producing a single, confident segmentation but rather a set of somewhat differing hypotheses across the lifted dimension.

The disagreement between $\mathbf{P}_{z_i}$ ($z_i \in S_{best}$) and $\mathbf{P}_{z'}$ ($z' \in S_{rest}$) will result in low pairwise similarity scores, leading to a low overall PQA score $Q$. Low similarity (low $Q$) reflects high variance among the internal predictions $\{\mathbf{P}_z\}_{z=1}^m$. In Bayesian deep learning and uncertainty estimation, predictive variance is commonly used as a proxy for model uncertainty. Therefore, a low $Q$, indicating high variance across slice predictions, correlates with higher model uncertainty and potentially lower prediction quality or reliability for that specific input. This assessment is intrinsic as it uses the internal structure of the model's output space without requiring multiple stochastic forward passes or external calibration models.
\end{proof}

\begin{figure*}[t]
  \centering
  \subfloat[Grad-CAM activations]{\includegraphics[height=7cm]{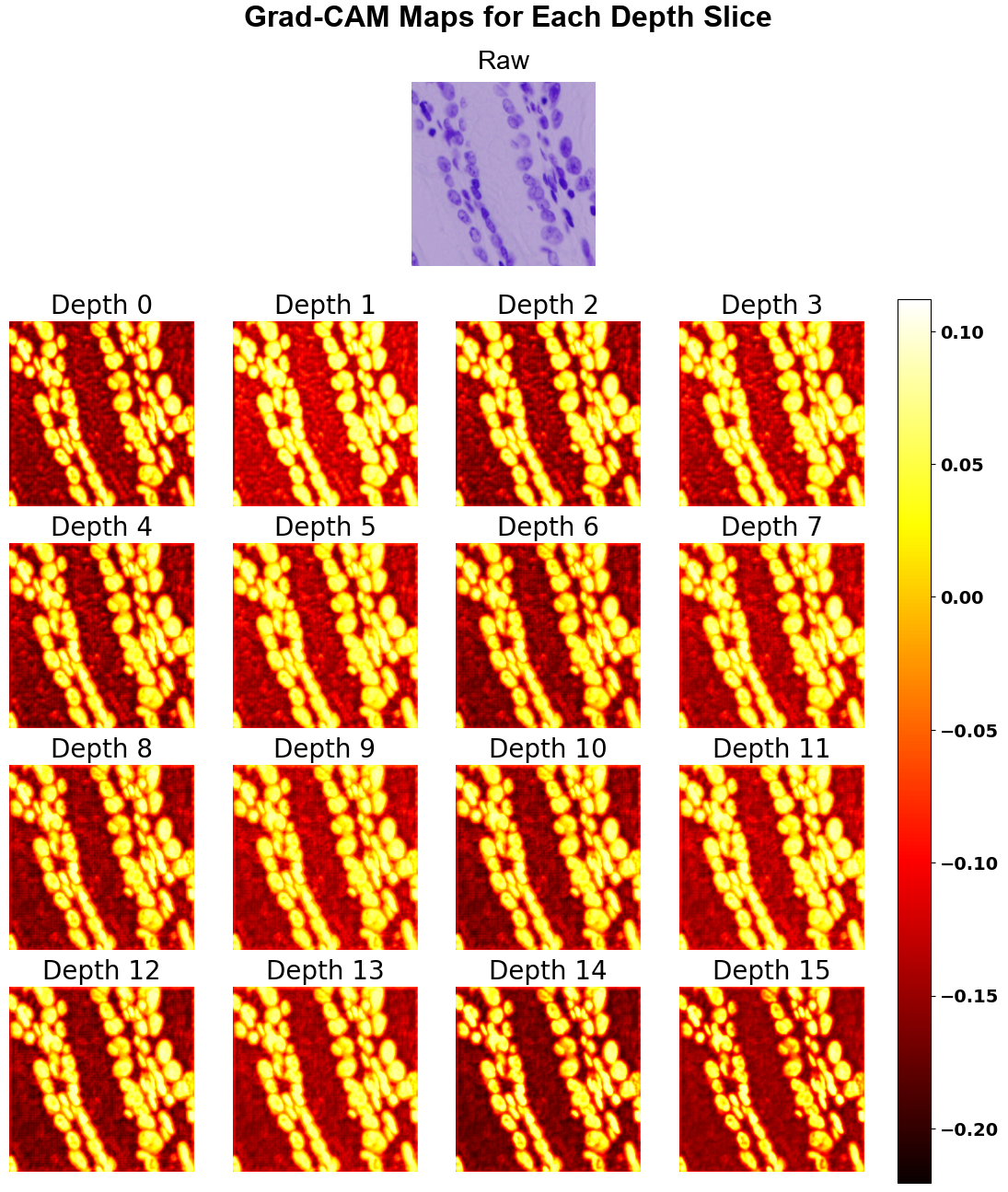}\label{fig:gradcam}}
  \hfill
  \subfloat[Saliency maps]{\includegraphics[height=7cm]{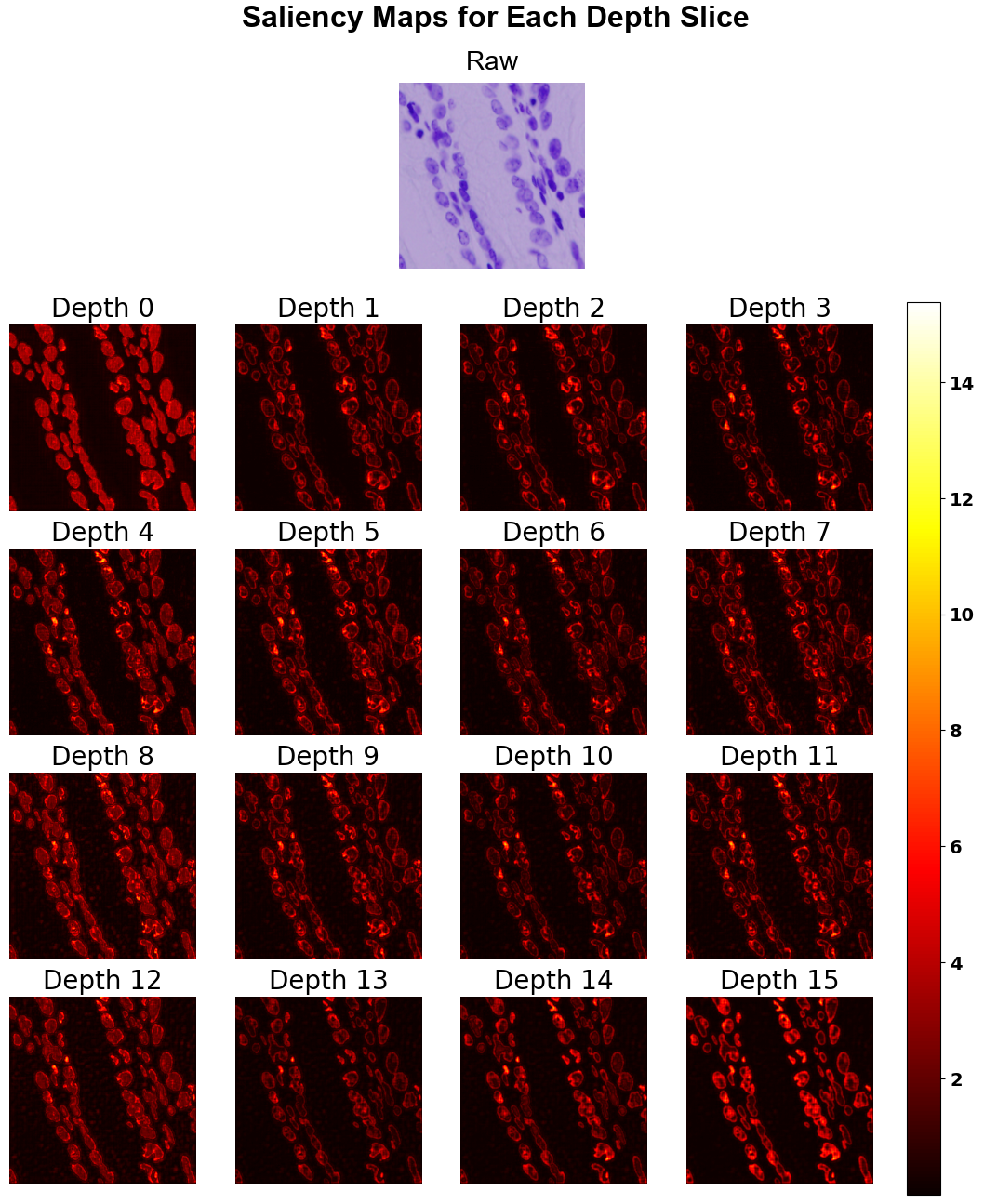}\label{fig:saliency}}
  \hfill
  \subfloat[Guided Backprop]{\includegraphics[height=7cm]{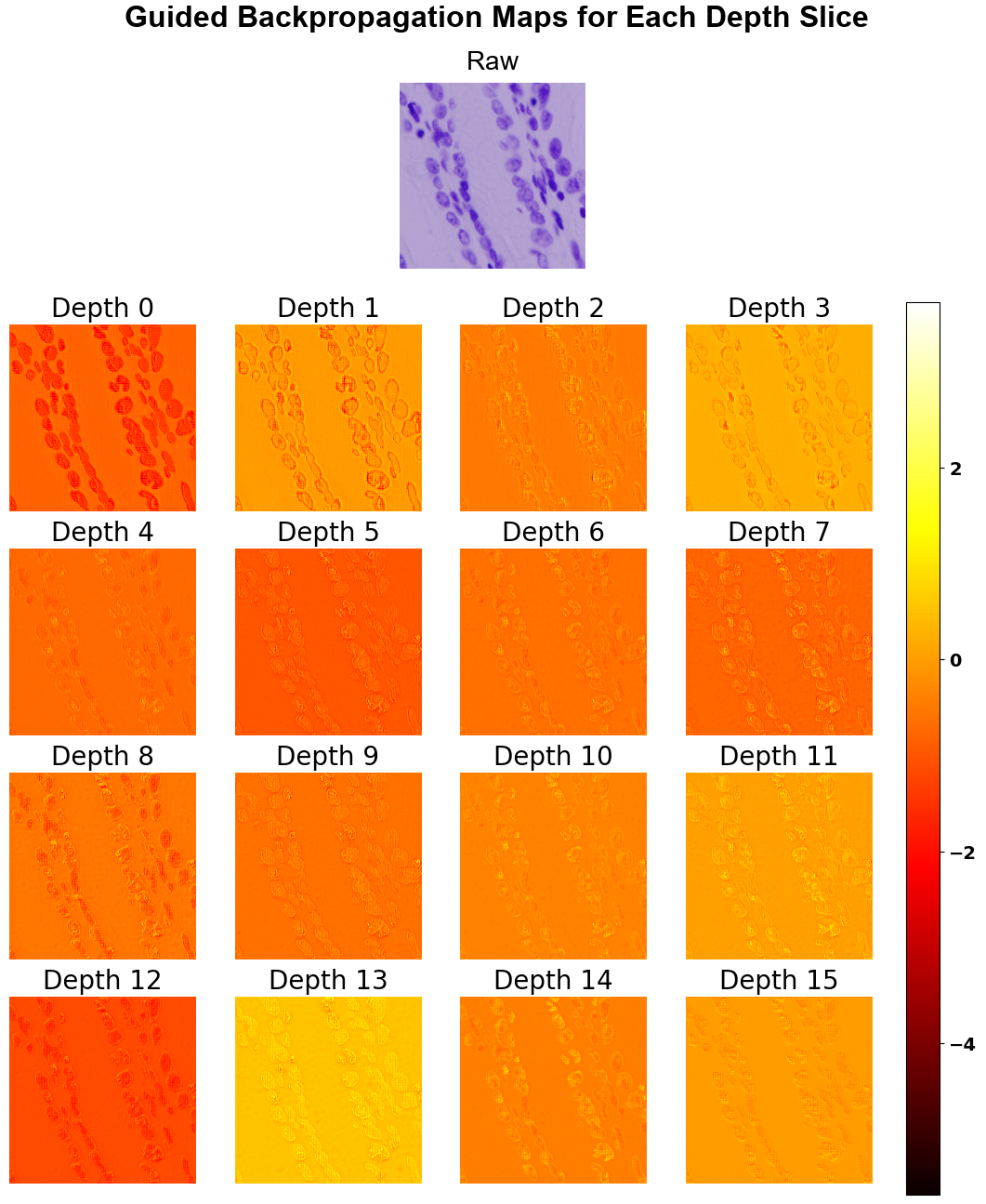}\label{fig:guidedbp}}
  \caption{Visualization results using different interpretability methods for each depth slice.}
  \label{fig:visualizations}
\end{figure*}

\subsection{Parameter and Computational Complexity}

We analyze the computational costs by comparing a conventional 2D dense prediction architecture against its Spatial Lifting (SL) counterpart. The core trade-off introduced by SL is substituting the deep, computationally expensive channels of 2D networks with an additional spatial dimension processed by lightweight, channel-shallow 3D operations.

\noindent \textbf{Conventional 2D Architecture.}
In a typical U-Net or FCN-style architecture, the number of channels $C_l$ at a layer $l$ often doubles with network depth (i.e., $C_l \propto 2^l$), leading to very wide feature maps in deeper layers (e.g., 512 or 1024 channels). For a 2D convolutional layer with a $K_x \times K_y$ kernel, the costs are:
\begin{itemize}
    \item {Parameters:} $\approx K_x \cdot K_y \cdot C_{l-1} \cdot C_l$
    \item {MACs:} $\approx H \cdot W \cdot K_x \cdot K_y \cdot C_{l-1} \cdot C_l$
\end{itemize}
This results in a parameter count that grows quadratically with the channel dimensions, which become the dominant source of model complexity.

\noindent \textbf{Spatial Lifting (SL) Architecture.}
The SL design lifts the input $m$ times and processes it with a 3D network that maintains a small, constant number of channels, $C_*$, throughout its depth (e.g., $C_*=8$). For a 3D convolutional layer with a $K_x \times K_y \times K_z$ kernel, the costs are:
\begin{itemize}
    \item {Parameters:} $\approx K_x \cdot K_y \cdot K_z \cdot C_*^2$
    \item {MACs:} $\approx H \cdot W \cdot m \cdot K_x \cdot K_y \cdot K_z \cdot C_*^2$
\end{itemize}

\noindent \textbf{Comparative Analysis.}
A direct comparison reveals the source of SL's efficiency.
\begin{enumerate}
    \item {Parameter Efficiency:} The SL architecture's parameter count is governed by $C_*^2$, where $C_*$ is a small constant. In contrast, the 2D model's parameters are driven by $C_{l-1} \cdot C_l$, which is orders of magnitude larger in deeper layers. This is why SL models can have over 95\% fewer parameters, as empirically demonstrated in Tables~\ref{Tab:Metrics} and \ref{Tab:arch_details}.

    \item {MACs Trade-off:} The MACs for SL introduce a multiplicative factor for the lifted dimension ($m$) and kernel depth ($K_z$). However, this increase is counteracted by the massive reduction from replacing the large $C_{l-1} \cdot C_l$ term with the very small $C_*^2$ term. When $C_*$ is sufficiently small (e.g., 8 or 16), the total MACs can be lower than in a conventional deep-channel 2D network.
\end{enumerate}
In essence, SL exchanges channel-wise complexity for manageable spatial complexity. The implicit regularization from the 3D convolutions across the lifted dimension allows the network to learn powerful representations without requiring hundreds of channels, leading to models that are not only smaller but also computationally efficient.

\section{Experiments}
Our experiments address two primary objectives. First, we investigate how the SL technique impacts performance and resource consumption (GMACs and model parameters) when applied to widely used architectures such as UNet~\cite{Ronneberger2015_UNet}. Second, we assess whether applying SL on vision decoders for processing features generated by pre-trained encoders (e.g., PVTv2~\cite{wang2022pvt}) can enhance both performance and efficiency. To thoroughly evaluate the proposed approach, our experiments focus on dense prediction tasks, encompassing not only semantic segmentation but also depth estimation, thereby demonstrating the versatility and effectiveness of SL across different applications. 

\begin{table*}[t]
\scriptsize
\centering
\caption{Details of datasets used for evaluation: name, task type, modality, training/testing splits, and image resolution.}
\resizebox{\textwidth}{!}{
\begin{tabular}{|c|c|c|c|c|c|}
\hline
\multirow{2}{*}{\textbf{Name}}& \multirow{2}{*}{\textbf{Task Type}}     &\multirow{2}{*}{\textbf{Modality}}                              &\multicolumn{2}{c|}{\textbf{Data Split}}&\multirow{2}{*}{\textbf{Resolution(px)}}\\ \cline{4-5}
                              &                                         &                                                                &\textbf{Training}     &\textbf{Testing}         &\\ \hline 
CHASE\_DB1                    &Retinal Vessel Segmentation              &Fundus Photography                                              &21                    &7                        &999$\times$960\\ \cline{1-6}             
FIVES                         &Retinal Vessel Segmentation              &Fundus Photography                                              &600                   &200                      &2048$\times$2048\\ \cline{1-6}
ISIC2018                      &Skin Lesion Segmentation                 &Dermoscopic Imaging                                             &2594                  &1000                     &Mixed\\ \cline{1-6}
Kvasir\text{-}SEG             &Colonic Polyp Segmentation               &Endoscopic Imaging                                              &880                   &120                      &Mixed\\ \cline{1-6}
GlaS                          &Gland Segmentation                       &H\&E Histopathological Imaging  &85                    &80                       &775$\times$522\\ \cline{1-6}
BCCD                          &Cell Segmentation                        &Microscopy Imaging                                              &1169                  &159                      &1600$\times$1200\\ \cline{1-6}
DSB2018                       &Nucleus Segmentation                     &Microscopy Imaging                                              &600                   &70                       &Mixed\\ \cline{1-6}
Fluorescent Neuronal Cells    &Neuronal Cell Segmentation               &Microscopy Imaging                                              &200                   &83                       &1600$\times$1200\\ \cline{1-6}
MoNuSAC                       &Nucleus Segmentation and Classification  &H\&E  Histopathological Imaging  &209                   &85                       &Mixed\\ \cline{1-6}
MoNuSeg                       &Nucleus Segmentation                     &H\&E  Histopathological Imaging  &37                    &14                       &1000$\times$1000\\ \cline{1-6}
NuInsSeg                      &Nucleus Segmentation                     &H\&E  Histopathological Imaging  &510                   &155                      &512$\times$512\\ \cline{1-6}
Sartorius                     &Cell Segmentation                        &H\&E  Histopathological Imaging  &540                   &66                       &704$\times$520\\ \cline{1-6}
TNBC                          &Nucleus Segmentation                     &H\&E  Histopathological Imaging  &36                    &14                       &512$\times$512\\ \cline{1-6}
Cityscapes                    &Depth Estimation                         &Urban Street Scenes        &2975                  &500                     &256$\times$128\\ \cline{1-6}
Make3D                        &Depth Estimation                         &Outdoor Landscape          &375                   &50                      &1704$\times$2272\\ \cline{1-6}
DIODE                         &Depth Estimation                         &Diverse Environments       &386                   &60                      &Mixed\\ \cline{1-6}
KITTI                         &Depth Estimation                         &Autonomous Driving Scenes  &800                   &200                     &1216$\times$352\\ \cline{1-6}
MODEST                        &Depth Estimation                         &Driving Scenes             &395k                  &5k                      &224$\times$224\\ \cline{1-6}
NYU Depth                     &Depth Estimation                         &Indoor Scenes              &50688                 &654                     &640$\times$480\\ \cline{1-6}
\end{tabular}
\label{Tab:Datasets}
}
\end{table*}

\begin{table}[t]
    \centering
    \footnotesize
    \caption{MACs and Parameters of UNet without and with SL.}
    \begin{tabular}{@{}c c c c c c@{}}
        \toprule
        \textbf{Metric} & \textbf{Model} & \textbf{5L, 1Res} & \textbf{8L, 1Res} & \textbf{5L, 2Res} & \textbf{8L, 2Res} \\
        \midrule
        \multirow{2}{*}{MACs (G)} & UNet & 13.7245 & 17.4800 & 18.5698 & 20.7387 \\
         & SL-UNet & \textbf{6.0206} & \textbf{6.0216} & \textbf{7.0719} & \textbf{7.0713} \\
        \midrule
        \multirow{2}{*}{Params (M)}& UNet  & 3.3522 & 33.0124 & 6.4959 & 43.2355 \\
        & SL-UNet & \textbf{0.0295} & \textbf{0.0371} & \textbf{0.0382} & \textbf{0.0464} \\
        \bottomrule
    \end{tabular}
    \label{Tab:Metrics}
\end{table}

\subsection{Datasets}
\noindent\textbf{For Semantic Segmentation,} we conduct experiments on thirteen publicly available datasets. These datasets encompass diverse anatomical structures, including the retina, skin, gastrointestinal tract, glands and various cell types, as well as multiple imaging modalities such as fundus photography, dermoscopy, histopathology, endoscopy, and microscopy. Specifically, the datasets are as follows: CHASE\_DB1~\cite{fraz2012ensemble}, DSB2018~\cite{Caicedo2019NucleusSA}, Kvasir\text{-}SEG~\cite{jha2020kvasir}, MoNuSAC~\cite{verma2021monusac2020}, ISIC2018~\cite{codella2018skin,tschandl2018ham10000},MoNuSeg~\cite{kumar2017dataset}, Sartorius~\cite{sartorius-cell-instance-segmentation}, TNBC~\cite{8438559}, Fluorescent Neuronal Cells~\cite{hitrec2019neural,morelli2021automating}, GlaS~\cite{sirinukunwattana2017gland,sirinukunwattana2015stochastic}, BCCD~\cite{Sarker2023}, FIVES~\cite{jin2022fives}, NuInsSeg~\cite{mahbod2023nuinsseg}. The CHASE\_DB1 and FIVES datasets are both used for retinal vessel segmentation. FIVES contains 800 multi-disease fundus photographs annotated at the pixel level, with 600 used for training and 200 for testing. For CHASE\_DB1, we select 21 images for training and 7 images for testing from a set of 28 images. ISIC2018 is a large-scale dataset for skin lesion segmentation, including 2,594 paired training images and 1,000 paired test images. Kvasir\text{-}SEG is an endoscopic dataset designed for pixel-level segmentation of colonic polyps. The dataset consists of 1,000 gastrointestinal polyp images and their corresponding segmentation masks, all annotated and verified by experienced gastroenterologists. We follow the official data split to divide the dataset into 880 images for training and 120 images for validation. GlaS is a gland segmentation dataset, it consists of 85 training images and 80 testing images, each focusing on segmenting glands in colorectal regions. The remaining datasets, including BCCD, DSB2018, Fluorescent Neuronal Cells, MoNuSAC, MoNuSeg, NuInsSeg, Sartorius, and TNBC, are specifically designed for cell or nucleus segmentation tasks. 

\noindent\textbf{For Depth Estimation,} we employ six publicly available datasets: Cityscapes~\cite{cordts2016cityscapes,liu2019end} , Make3D~\cite{saxena2005learning,saxena2007learning}, DIODE~\cite{vasiljevic2019diode}, KITTI~\cite{geiger2013vision}, NYU Depth~\cite{silberman2012indoor} and MODEST~\cite{shan2021modestkaggle}. These datasets cover a wide range of scenes, from outdoor urban environments to indoor scenarios, providing diverse depth cues and challenging variations in lighting, texture, and geometry. Such diversity ensures comprehensive evaluation of model performance under various real-world conditions. The outdoor urban scene datasets include Cityscapes, KITTI, DIODE and Make3D. Cityscapes contains high-resolution images captured in urban street scenes, providing stereo image pairs with dense depth annotations. we use 2,975 images for training and 500 images for testing from the official Cityscapes split. KITTI, collected from a moving vehicle in urban and highway settings, offers stereo RGB images with LiDAR-based ground truth depth, we utilize a subset of 1,000 images originally designated for validation, which we further split into 800 images for training and 200 images for testing to evaluate our model. Make3D provides single outdoor images and corresponding laser-scanned depth maps, containing 425 images divided into 375 training and 50 testing samples. The DIODE dataset comprises RGB-D images captured by a handheld sensor in both indoor and outdoor environments. In this work, only the outdoor subset is utilized, with 80\% of the data used for training and 20\% for testing. The indoor scene datasets include NYU Depth V2 and MODEST. NYU Depth V2 is captured with a Microsoft Kinect sensor and consists of RGB-D images from various indoor scenes, with 795 training images and 654 test images according to the official partition. MODEST is a large-scale synthetic dataset containing approximately 400,000 images, generated by compositing foreground objects onto diverse background scenes in museum environments. It is designed specifically for monocular depth estimation and segmentation, offering complex indoor layouts with high-quality depth annotations. {More details on the datasets used for the semantic segmentation and depth estimation experiments can be found in Table~\ref{Tab:Datasets}.}

\begin{table*}[t]
\footnotesize
\centering
\caption{Performance of UNet and SL-UNet on Segmentation Datasets (Dice Score \%)}
\begin{tabular}{|c|c|c|c|c|c|c|c|}
\hline
\textbf{Config} & \textbf{Model} & \textbf{DSB2018} & \textbf{Kvasir-SEG} & \textbf{ISIC2018} & \textbf{Sartorius} & \textbf{MoNuSAC} & \textbf{C\_DB1} \\ \hline
\multirow{2}{*}{5L, 1Res}
& UNet    & 87.27 & 52.62 & 80.22 & 61.77 & 62.73 & 59.72 \\
& SL-UNet & 87.82 & 52.35 & 78.81 & 69.85 & 66.21 & 66.14 \\ \hline
\multirow{2}{*}{5L, 2Res}
& UNet    & 87.76 & 71.73 & 75.33 & 70.92 & 61.99 & 60.71 \\
& SL-UNet & 90.00 & 69.50 & 82.46 & 70.93 & 66.61 & 66.10 \\ \hline
\multirow{2}{*}{8L, 1Res}
& UNet    & 87.96 & 64.68 & 79.58 & 61.90 & 62.92 & 58.87 \\
& SL-UNet & 87.97 & 57.49 & 85.28 & 70.80 & 66.76 & 66.57 \\ \hline
\multirow{2}{*}{8L, 2Res}
& UNet    & 87.95 & 82.69 & 78.91 & 69.10 & 62.08 & 59.68 \\
& SL-UNet & 88.61 & 76.04 & 84.96 & 71.57 & 66.40 & 67.62 \\ \hline
\end{tabular}

\vspace{3mm} 

\begin{tabular}{|c|c|c|c|c|c|c|c|c|}
\hline
\textbf{Config} & \textbf{Model} & \textbf{MoNuSeg} & \textbf{TNBC} & \textbf{FNC} & \textbf{GlaS} & \textbf{BCCD} & \textbf{FIVES} & \textbf{NuInsSeg} \\ \hline
\multirow{2}{*}{5L, 1Res}
& UNet    & 70.79 & 57.84 & 55.89 & 75.27 & 93.71 & 71.82 & 59.50 \\
& SL-UNet & 73.20 & 62.93 & 65.60 & 76.27 & 95.25 & 73.05 & 60.76 \\ \hline
\multirow{2}{*}{5L, 2Res}
& UNet    & 70.37 & 63.00 & 55.85 & 76.23 & 93.99 & 71.97 & 62.21 \\
& SL-UNet & 73.17 & 63.08 & 61.47 & 80.49 & 94.36 & 74.57 & 67.54 \\ \hline
\multirow{2}{*}{8L, 1Res}
& UNet    & 70.83 & 59.77 & 57.32 & 77.86 & 93.80 & 70.99 & 62.51 \\
& SL-UNet & 72.86 & 62.52 & 64.19 & 78.53 & 94.13 & 71.28 & 63.23 \\ \hline
\multirow{2}{*}{8L, 2Res}
& UNet    & 70.62 & 67.11 & 57.50 & 78.91 & 93.75 & 72.64 & 62.67 \\
& SL-UNet & 73.53 & 65.24 & 63.78 & 79.65 & 94.65 & 73.26 & 66.80 \\ \hline
\end{tabular}
\label{Tab:UNet_SLUNet}
\end{table*}
\begin{table}[t]
\footnotesize
\centering
\caption{Pearson correlation (\(r\)) and Spearman rank correlation (\(\rho\)) between the predicted Dice scores from SL-UNet (8L, 2Res) and the actual Dice scores.}
\begin{footnotesize}
\begin{tabular}{|c|c|c|c|}
\hline
\textbf{Dataset} & \textbf{Pearson $r$} & \textbf{Spearman $\rho$} & \textbf{$p$-value} \\ \cline{1-4}           
CHASE\_DB1 & 0.8313 & 0.8571 & \textbf{< 0.05} \\ \cline{1-4}
DSB2018 & 0.6189 & 0.5271 & \textbf{< 0.05} \\ \cline{1-4}
Kvasir\text{-}SEG & 0.7107 & 0.7495 & \textbf{< 0.05} \\ \cline{1-4}
MoNuSAC & 0.6637 & 0.6327 & \textbf{< 0.05} \\ \cline{1-4}
ISIC2018 & 0.7750 & 0.6230 & \textbf{< 0.05} \\ \cline{1-4}
MoNuSeg & 0.4598 & 0.1604 & > 0.05 \\ \cline{1-4}
Sartorius & 0.6462 & 0.3681 & \textbf{< 0.05} \\ \cline{1-4}
TNBC & 0.3855 & 0.6264 & \textbf{< 0.05} \\ \cline{1-4}
FNC & 0.4268 & 0.4128 & \textbf{< 0.05} \\ \cline{1-4}
GlaS & 0.7410 & 0.8045 & \textbf{< 0.05} \\ \cline{1-4}
BCCD & 0.6259 & 0.5895 & \textbf{< 0.05} \\ \cline{1-4}
FIVES & 0.9199 & 0.6543 & \textbf{< 0.05} \\ \cline{1-4}
NuInsSeg & 0.6335 & 0.4326 & \textbf{< 0.05} \\ \cline{1-4}
\end{tabular}
\end{footnotesize}
\label{Tab:correlation}
\end{table}

\subsection{Spatial Lifting on UNet}\label{sec:SLU}
Table~\ref{Tab:Metrics} compares the model parameter count and inference cost (in GMACs) of the original UNet and SL-UNet. The notation ``5L'' denotes a five-layer UNet, while ``1Res'' indicates the inclusion of one residual connection. Both implementations use the Monai package. In the conventional five-layer U-Net architecture, the number of feature channels is typically configured as 32, 64, 128, 256, and 512 from the first to the fifth layer, whereas SL-UNet adopts a fixed channel setting of 8 and lifts the 2D input to 3D with a depth of 16. SL-UNet features an extremely small number of model parameters and lower GMACs, with extensive parameter sharing across the 3D operations, resulting in a highly lightweight design.

Table~\ref{Tab:UNet_SLUNet} presents the segmentation performance of UNet and SL-UNet across 13 benchmark datasets. SL-UNet consistently outperforms the original UNet on most datasets, achieving better segmentation accuracy while maintaining substantially lower parameter counts and computational costs (cross-referencing Table~\ref{Tab:Metrics}). The observed performance gains underscore the effectiveness and efficiency of the SL design. Notably, given the minimal parameter count of SL-UNet, this suggests a new phenomenon in dense prediction modeling, emphasizing the potential for achieving strong performance with exceptionally lightweight model architectures.

Additionally, Table~\ref{Tab:correlation} demonstrates the prediction quality assessment of the SL technique. SL-UNet produces predicted Dice scores that exhibit moderate to high correlations with the actual Dice scores across most datasets, further validating its utility for reliable AI deployment. To better understand how different depth slices contribute to the final prediction, visualization results using various interpretability methods are shown in Fig.~\ref{fig:visualizations}. Specifically, Grad-CAM (a), saliency maps (b), and guided backpropagation (c) highlight the informative regions captured by each slice within the SL decoder.

\begin{table}[t]
\scriptsize
\centering
\caption{Training configurations: semantic segmentation (cosine-annealed learning rate, weighted BCE+IoU losses) and depth estimation (multi-step decay, masked MSE+L1 losses). Common: AdamW optimizer.}
\label{Tab:train_details}
\begin{tabular}{@{}lcc@{}}
\toprule
\textbf{Configuration} & \textbf{Segmentation Models} & \textbf{Depth Estimation Models} \\
\midrule
Epoch                 & 40                          & 100                             \\
Optimizer             & AdamW                       & AdamW                          \\
LearningRate          & 0.001                       & 0.001                          \\
Scheduler             & Cosine Annealing Warm Restarts & Stepwise Learning Rate decay \\
Clip                  & 0.5                         & 0.5                            \\
LossFunction          & WBCE + WIoU                 & MaskedMSE + MaskedL1           \\
\bottomrule
\end{tabular}
\end{table}

\begin{table}[t]
\scriptsize
\centering
\caption{Decoder architectures: channel dimensions, depth, and complexity metrics.}
\begin{tabular}{@{}lccccccc@{}}
\toprule
\textbf{Decoder} & \textbf{Channel} & \textbf{Depth} & \textbf{MACs (G)} & \textbf{Params (M)} \\
\midrule
HSNet          & [64,128,320,512] & -- & 1.5653 & 4.4072\\
SL-UNet(5L, 1Res) & [8,8,8,8,8,8,8,8] & 16 & \textbf{1.2653} & \textbf{0.5414} \\
SL-UNet(8L, 1Res) & [8,8,8,8,8,8,8,8] & 16 & 1.2835 & 0.5648\\
SL-UNet(5L, 2Res) & [8,8,8,8,8,8,8,8] & 16 & 1.3296 & 0.5762\\
SL-UNet(8L, 2Res) & [8,8,8,8,8,8,8,8] & 16 & 1.3499 & 0.6020 \\ \hline
EMCAD & [64,128,320,512] & -- & 1.1318 & 1.9135\\
SL-EMCAD & [16,32,64,128] & [8,4,2,1] & \textbf{1.0467} & \textbf{0.2623}\\ \hline
CASCADE & [64,128,320,512] & -- & 6.8730 & 10.4219\\
SL-CASCADE & [16,32,64,128] & [8,4,2,1] & \textbf{4.7814} & \textbf{1.6409}\\ \hline
FastDepth & [1280,16,24,32] & -- & 0.2311 & 0.1278\\
SL-FastDepth & [16,8,8,8] & 16 & \textbf{0.2248} & \textbf{0.0404}\\ \hline
MiDaS & [256,256,256,256] & -- & 15.4028 & 8.5934\\
SL-MiDaS & [16,16,16,16] & 16 & \textbf{5.3586} & \textbf{0.1055}\\
\bottomrule
\end{tabular}
\label{Tab:arch_details}
\end{table}

\begin{figure*}[t]
  \noindent
  \begin{tabular}{@{} >{\centering\arraybackslash}m{1cm} m{\dimexpr\textwidth-1cm}@{}}
    (a) & \parbox[t]{\dimexpr\textwidth-1cm}{
          \centering
          \includegraphics[width=0.9\textwidth]{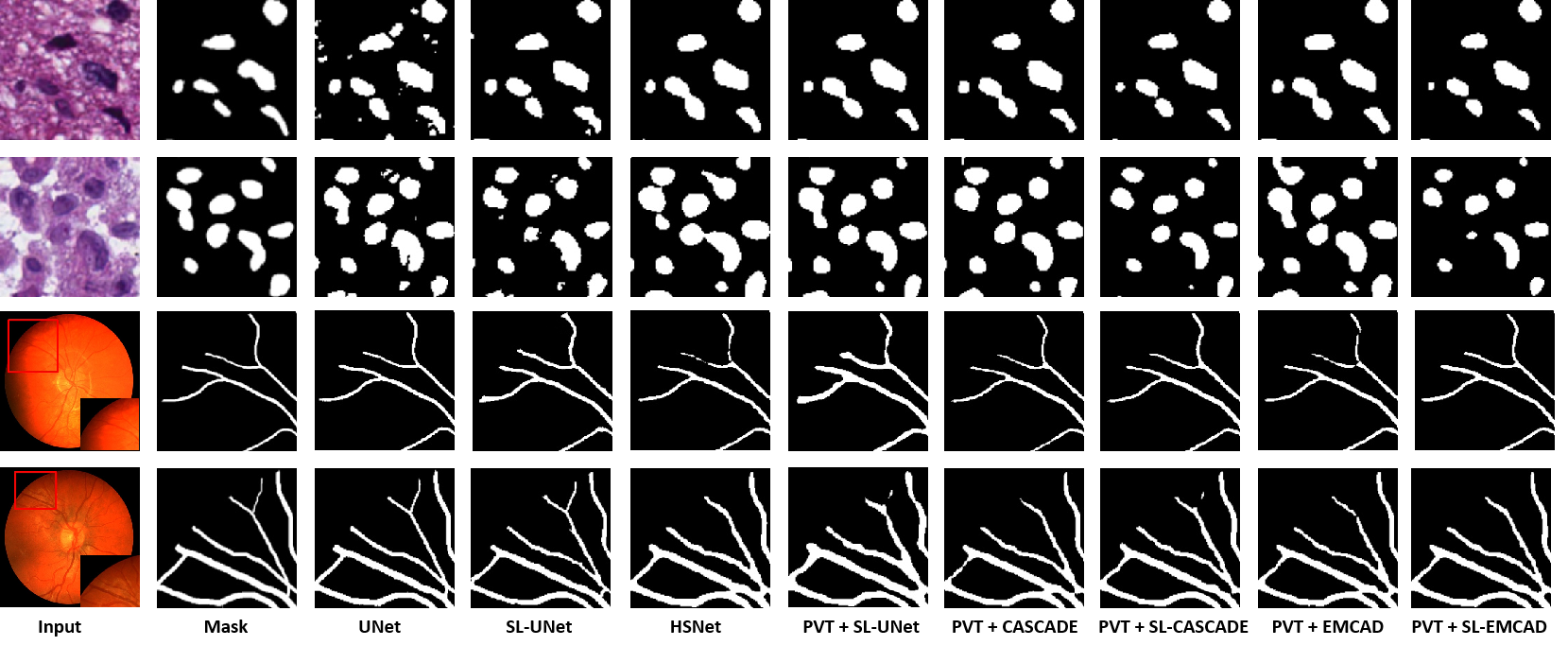}
        } \\
    \noalign{\vskip 0.5cm}
    (b) & \parbox[t]{\dimexpr\textwidth-1cm}{
          \centering
          \includegraphics[width=0.7\textwidth]{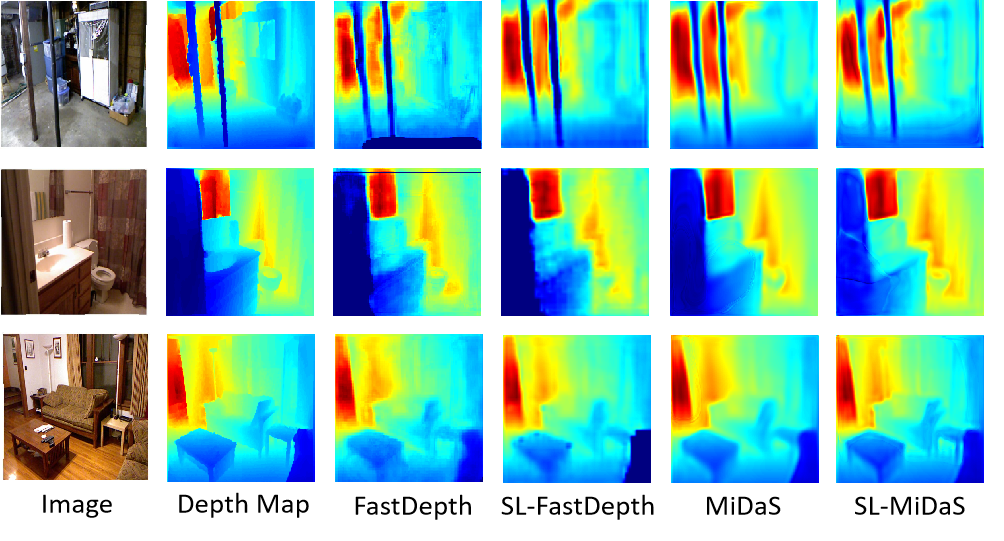}
        } \\
  \end{tabular}
  \caption{(a) Semantic segmentation results; (b) Depth estimation results.}
  \label{fig:rendering}
\end{figure*}

\subsection{Spatial Lifting on Vision Decoders}
To leverage the common practice of employing a pre-trained encoder with a decoder for dense prediction tasks, we utilize a pre-trained PVTv2 as the encoder and incorporate the SL technique to enhance the decoder’s effectiveness in dense prediction. For performance comparison, we include HSNet’s decoder~\cite{ZHANG2022106173}, the recently proposed state-of-the-art efficient decoder EMCAD~\cite{rahman2024emcad}, and its SL-enhanced variant, SL-EMCAD. All models were trained for 40 epochs using the AdamW optimizer with a learning rate of 0.001 and a Cosine Annealing Warm Restarts (CAWR) learning rate scheduler. A gradient clipping value of 0.5 was applied, and the loss functions included BCE with IoU for most configurations, while EMCAD and SL-EMCAD utilized weighted BCE (wbce) combined with weighted IoU (wiou). Multi-scale training was applied across all models using scale factors of 0.75, 1.0, and 1.25.(see Table~\ref{Tab:train_details}).

In Table~\ref{Tab:arch_details}, we present the GMACs and parameter counts of various vision decoders for dense prediction. Compared to state-of-the-art efficient decoders, SL-UNet exhibits slightly higher GMACs but significantly lower parameter counts. Notably, when applying SL to EMCAD, both GMACs and parameter counts are further reduced. With these efficiency metrics in mind, Table~\ref{tab:hsnet_vs_pvt_emcad} shows that compared to the baseline decoder architecture, the introduction of SL leads to consistent performance improvements across most datasets. With the increase in the number of decoder layers and the addition of residual connections, the model performance is further strengthened, reaching its peak with the “8L + 2Res” configuration. This validates the complementary roles of the proposed method, skip connections, and residual structures in capturing both global semantic context and fine-grained spatial details. To further assess the generalizability of the SL module, we incorporated it into two representative state-of-the-art decoder architectures, CASCADE~\cite{rahman2023medical} and EMCAD. While adaptation of the original architectures was required, the integration of SL consistently yielded stable performance improvements across the majority of evaluated datasets. This demonstrates that SL is a versatile enhancement module that can be adapted to and effectively integrated within various decoder designs.
Visual results for semantic segmentation are presented in Fig.~\ref{fig:rendering} (a).
\begin{table*}[t]
\scriptsize
\centering
\caption{Spatial Lifting on vision decoders. Performance evaluated by Dice Score (\%).}

\begin{tabular}[t]{|c|c|c|c|c|c|c|c|}
\hline
\textbf{Model} & \textbf{DSB2018} & \textbf{Kvasir-SEG} & \textbf{ISIC2018} & \textbf{Sartorius} & \textbf{MoNuSAC} & \textbf{CHASE\_DB1} \\
\hline
HSNet & 86.09 & 68.39 & 79.27 & 65.28 & 59.04 & 60.92 \\
PVTv2 + UNet(8L,2Res) & 89.60 & \textbf{84.11} & 86.28 & 70.51 & 63.82 & 64.06 \\
PVTv2 + SL-UNet(5L,1Res) & 88.65 & 76.95 & 80.28 & 67.06 & 65.19 & 63.46 \\
PVTv2 + SL-UNet(5L,2Res) & 88.55 & 81.75 & 81.41 & \textbf{71.80} & 66.48 & 64.37 \\
PVTv2 + SL-UNet(8L,1Res) & 87.71 & 82.41 & 80.35 & 68.39 & 66.43 & 64.51 \\
PVTv2 + SL-UNet(8L,2Res) & \textbf{90.43} & 81.20 & \textbf{87.07} & 68.57 & \textbf{66.72} & \textbf{67.07} \\
\hline
PVTv2 + CASCADE & \textbf{90.48} & 65.96 & 77.32 & 68.29 & \textbf{63.25} & \textbf{71.58} \\
PVTv2 + SL-CASCADE & 89.92 & \textbf{66.43} & \textbf{78.72} & \textbf{72.98} & 62.79 & 70.17\\
\hline
PVTv2 + EMCAD & \textbf{90.56} & \textbf{70.96} & 78.41 & 63.26 & 59.89 & \textbf{70.52} \\
PVTv2 + SL-EMCAD & 87.38 & 64.48 & \textbf{79.93} & \textbf{66.62} & \textbf{60.37} & 69.47 \\
\hline
\end{tabular}

\vspace{3mm} 

\begin{tabular}[t]{|c|c|c|c|c|c|c|c|}
\hline
\textbf{Model} & \textbf{MoNuSeg} & \textbf{TNBC} & \textbf{FNC} & \textbf{GlaS} & \textbf{BCCD} & \textbf{FIVES} & \textbf{NuInsSeg} \\
\hline
HSNet & 69.84 & 65.75 & 55.42 & 79.79 & 92.35 & 71.03 & 65.94 \\
PVTv2 + UNet(8L,2Res)    & 73.81 & 62.51 & 56.22 & 84.70 & 94.64 & 76.36 & 66.62\\
PVTv2 + SL-UNet(5L,1Res) & 74.61 & 68.58 & 56.06 & 82.33 & 94.67 & 74.04 & 65.51 \\
PVTv2 + SL-UNet(5L,2Res) & 74.17 & 66.66 & 56.50 & 84.66 & 94.97 & 74.51 & \textbf{68.05} \\
PVTv2 + SL-UNet(8L,1Res) & 74.52 & 70.66 & 57.36 & 82.59 & 95.01 & 73.47 & 63.78 \\
PVTv2 + SL-UNet(8L,2Res) & \textbf{76.31} & \textbf{70.72} & \textbf{57.37} & \textbf{87.05} & \textbf{95.63} & \textbf{76.69} & 65.34 \\
\hline
PVTv2 + CASCADE & 75.16 & 59.61 & 51.04 & 83.17 & \textbf{95.82} & \textbf{80.24} & 69.08 \\
PVTv2 + SL-CASCADE & \textbf{76.86} & \textbf{64.83} & \textbf{51.54} & \textbf{85.08} & 94.89 & 77.56 & \textbf{69.99}\\
\hline
PVTv2 + EMCAD    & 75.86 & 65.09 & \textbf{61.75} & 79.84 & \textbf{95.36} & \textbf{80.30} & \textbf{70.34} \\
PVTv2 + SL-EMCAD & \textbf{76.31} & \textbf{67.97} & 60.19 & \textbf{84.13} & 95.14 & 77.72 & 67.51\\
\hline
\end{tabular}
\label{tab:hsnet_vs_pvt_emcad}
\end{table*}

\subsection{Spatial Lifting for Depth Estimation}
For depth estimation, we employed two encoder architectures, MobileNet~\cite{sandler2018mobilenetv2} and ResNeXt~\cite{xie2017aggregated}, in conjunction with two types of decoders, FastDepth's~\cite{icra_2019_fastdepth} and MiDaS's~\cite{ranftl2020towards}. The SL module was further integrated into these configurations to investigate its compatibility and potential utility across diverse network architectures. All models underwent training for 100 epochs using the AdamW optimization method with an initial learning rate of 0.001. To promote stable convergence during training, a stepwise learning rate decay schedule was utilized, and gradient clipping was applied with a maximum norm of 0.5. Distinct from semantic segmentation, the depth estimation task utilized loss functions tailored to the characteristics of sparse ground truth, specifically Masked L1 Loss and Masked Mean Squared Error (MSE) Loss, to effectively exclude invalid or missing depth regions during optimization.(see Table~\ref{Tab:train_details})

Table~\ref{Tab:depth estimation} details a comparative analysis of model performance with and without the integration of the SL module. The models augmented with SL exhibit a substantial decrease in both multiply-accumulate operations (MACs) and parameter count compared to the baseline architectures(see Table~\ref{Tab:arch_details}), while simultaneously achieving significant improvements in accuracy across the majority of the evaluated datasets. For instance, when adopting the ResNeXt encoder in conjunction with the MiDaS's decoder,  the MACs decreased by approximately 65.3\%, and the number of parameters was reduced to nearly 1.2\% of the original count. Moreover, under the aforementioned encoder–decoder configuration, except for the KITTI and Make3D, the Root Mean Square Error (RMSE) exhibited consistent reductions across all evaluated datasets, while the $\delta_1$ accuracy metric demonstrated improvements uniformly across all datasets. These results underscore the efficacy of the SL module not only within semantic segmentation tasks but also its promising applicability to dense prediction problems such as depth estimation. Visual results for depth estimation are presented in Fig.~\ref{fig:rendering} (b).

\begin{table*}[t]
\footnotesize
\centering
\caption{Comparison of decoders in depth estimation performance.}
\resizebox{\textwidth}{!}{%
\begin{tabular}{ccccccccccccc}
\toprule
\multirow{2}{*}{\textbf{Model}} & \multicolumn{2}{c}{\textbf{Cityscapes}}& \multicolumn{2}{c}{\textbf{Make3D}}& \multicolumn{2}{c}{\textbf{DIODE}} & \multicolumn{2}{c}{\textbf{KITTI}} & \multicolumn{2}{c}{\textbf{NYU}} & \multicolumn{2}{c}{\textbf{MODEST}} \\
\cmidrule(lr){2-3} \cmidrule(lr){4-5} \cmidrule(lr){6-7} \cmidrule(lr){8-9} \cmidrule(lr){10-11} \cmidrule(lr){12-13}
& {RMSE $\downarrow$} & \textbf{$\delta_1$ $\uparrow$} & {RMSE $\downarrow$} & \textbf{$\delta_1$ $\uparrow$} & {RMSE $\downarrow$} & \textbf{$\delta_1$ $\uparrow$} & {RMSE $\downarrow$} & \textbf{$\delta_1$ $\uparrow$} & {RMSE $\downarrow$} & \textbf{$\delta_1$ $\uparrow$} & {RMSE $\downarrow$} & \textbf{$\delta_1$ $\uparrow$}\\
\midrule      
MobileNet + FastDepth  
& 0.876& 0.576& 1.485 & 0.311& 0.572 & 0.217 & 2.144& 0.220& \textbf{1.449} & 0.450 & \textbf{0.484} & \textbf{0.673} \\
MobileNet + SL-FastDepth 
& \textbf{0.870} & \textbf{0.588} & \textbf{1.434} & \textbf{0.330} & \textbf{0.567} & \textbf{0.236} & \textbf{2.137} & \textbf{0.233} & 1.559 &\textbf{0.466} & 0.507 & 0.651 \\ \hline
ResNeXt + MiDaS   
& 0.918 & 0.556 & \textbf{1.457} & 0.336 & 0.555 & 0.271 & \textbf{2.114} & 0.239 & 1.501 & 0.444 & 0.450 & 0.752\\      
ResNeXt + SL-MiDaS 
& \textbf{0.811} & \textbf{0.620} & 1.690 & \textbf{0.342} & \textbf{0.528} & \textbf{0.287} & 2.185 & \textbf{0.256} & \textbf{1.333} & \textbf{0.486} & \textbf{0.439} & \textbf{0.764} \\ 
\bottomrule
\end{tabular}%
}
\label{Tab:depth estimation}
\end{table*}

\section{Limitations and Areas for Improvement}
Not all experiments across the datasets yielded improved predictive performance (e.g., higher Dice scores) for SL. We attribute this to several factors. First, spatial lifting in this study involves expanding the original input to a higher-dimensional representation with a fixed additional dimension (e.g., depth = 16). However, for different tasks and datasets, this value should ideally be adjustable. Neural architecture search methods could potentially be employed to determine the optimal dimensional setting, but due to space and time constraints, this aspect was not explored.

Additionally, in this work, SL utilizes an off-the-shelf higher-dimensional model (e.g., 3D-UNet). This presents a potential area of improvement, as a more specialized higher-dimensional model could be developed specifically for SL-type inputs and outputs. Moreover, given the potential redundancy in processing the lifted inputs with the current high-dimensional model, further reductions in GMACs (or GFLOPs) may be achievable through architecture optimization tailored to SL.

\section{Conclusion}
In this work, we developed Spatial Lifting (SL), a novel methodology for dense prediction. SL operates by lifting low-dimensional inputs, such as 2D images, into a higher-dimensional space and processing them with deep networks designed for that dimension, like a 3D U-Net. Our extensive experiments on a wide array of semantic segmentation and depth estimation tasks demonstrated that this approach, contrary to what one might expect from increasing dimensionality, leads to a remarkable reduction in model parameters and computational costs. Despite this enhanced efficiency, SL-based models consistently achieved performance comparable or superior to traditional, more complex 2D architectures. The effectiveness of SL is supported by our theoretical analysis, which suggests that the method introduces a beneficial inductive bias through implicit spatial regularization, potentially leading to improved generalization. In addition, a key advantage of the SL framework is its intrinsic ability to produce structured outputs along the lifted dimension. We showed how this structure can be harnessed to provide a reliable, near-zero-cost assessment of prediction quality at test time by measuring the consistency across output slices—a critical feature for deploying models in real-world applications.

Future work could explore adaptive lifting strategies and the design of custom higher-dimensional networks specifically tailored for the SL paradigm, which may lead to greater efficiency and performance gains. Overall, Spatial Lifting presents a new vision modeling paradigm, shifting the focus from architectural complexity in the native input spatial dimension to the strategic use of higher-dimensional spatial representations. This work suggests a promising avenue toward developing more efficient, accurate, and self-aware deep learning models for dense prediction problems.

\bibliographystyle{IEEEtran}
\bibliography{reference}

\end{document}